\documentclass[10pt,twocolumn,letterpaper]{article}

\usepackage{dx}
\usepackage[table]{xcolor}
\usepackage{times}
\usepackage{epsfig}
\usepackage{array}
\usepackage{graphicx}
\usepackage{amsmath,multirow}
\usepackage{amssymb,amsthm,comment}
\usepackage{algpseudocode,caption}
\usepackage[linesnumbered,ruled]{algorithm2e}
\newtheorem{theorem}{Theorem}

\usepackage[pagebackref=true,breaklinks=true,letterpaper=true,colorlinks,bookmarks=false]{hyperref}

\begin{document}

\title{Generative Adversarial Network {B}ased Autoencoder: Application to fault detection problem for closed-loop dynamical systems\thanks{This work was supported by the U.S. Department of Energy's Building Technologies Office through the Emerging Technologies, Sensors and Controls Program.}}

\author%
{%
I. Chakraborty$^{1,\diamond}$ ,
R. Chakraborty$^{2,\diamond}$ \and
D. Vrabie$^1$
\\
$^1$Optimization and Control Group\\
Pacific Northwest National Laboratory, Richland, WA, USA\\
$^2$CVGMI, University of Florida, FL, USA\\
\textsuperscript{1}{\tt\small \{indrasis.chakraborty, draguna.vrabie\}@pnnl.gov} ~~~~\textsuperscript{2}{\tt\small rudrasischa@gmail.com}\\
{\small $^\diamond$ These authors contributed equally to this work}
}

\maketitle

\begin{abstract}
The fault detection problem for closed-loop, uncertain dynamical systems is investigated in this paper, using different deep-learning based methods. The traditional classifier-based method does not perform well, because of the inherent difficulty of detecting system-level faults for a closed-loop dynamical system. Specifically, the acting controller in any closed-loop dynamical system works to reduce the effect of system-level faults. A novel generative-adversarial-based deep autoencoder is designed to classify data sets under normal and faulty operating conditions. This proposed network performs quite well when compared to any available classifier-based methods, and moreover, does not require labeled fault-incorporated data sets for training purposes. This network's performance is tested on a high-complexity building energy system data set.
\end{abstract}

\section{Introduction}
\label{intro}
Fault detection and isolation enables safe operation of critical dynamical systems, along with cost effective system performance and maximally effective control performance. For this reason, fault detection and isolation research  is of interest in many engineering areas, such as aerospace systems (e.g., \cite{patton1991fault,patton1992robust,patton1994review,zhang2015bayesian}), automotive systems (e.g., \cite{dixon2000fault,murray2002resolver,capriglione2003line,isermann2004model,mcintyre2004fault,hwang2010survey,flett2016fault}), photovoltaic systems (e.g., \cite{firth2010simple,chouder2010automatic,braun2012signal,zhao2012decision,yang2013benchmarking,silvestre2013automatic,garoudja2017statistical}),\nocite{chouder2010automatic,garoudja2017statistical,ropp1999prevention,yang2013benchmarking,firth2010simple,braun2012signal,zhao2012decision} and building heating and cooling systems (e.g., \cite{du2014fault,gao2016system}). 
For feedback-controlled dynamical systems subjected to exogenous disturbances, fault detection and isolation becomes challenging because the controller expends effort to compensate for the undesired effect of the fault. 

In this paper, we will focus on the fault detection problem. The objective will be to successfully distinguish data sets collected under faulty operating conditions from data sets representative of normal operating conditions. We will only investigate physical faults that affect the system dynamics. One can classify the approaches to fault detection problems based on the assumption regarding system dynamics, namely linear or nonlinear systems, and based on the use of a system model for fault detection, either model-driven or data-driven. Model-driven methods use a model for the dynamical system to detect the fault, whereas the data-driven methods do not make explicit use of a model of the physical system. Next we provide a brief overview of the available literature in all these categories.\par 

The fault detection problem for linear systems was first formulated in \cite{beard1971failure} and \cite{jones1973failure}. Both papers developed Luenberger-observer based approaches, where the observer gain matrix decouples the effects of different faults. The observer-based approach was extended in \cite{massoumnia1989failure} to include fault identification by solving the problem of residual generation by processing the inputs and outputs of the system. A model- and parameter-estimation based fault detection method is developed in \cite{frank1990fault}. An observer-based fault detection approach, where eigenstructure assignment provides robustness to the effects of exogenous disturbances, is demonstrated in \cite{patton1997observer}. Sliding-mode observers are used in \cite{edwards2000sliding} and \cite{tan2003sliding}, who also provide fault severity estimates. Isermann and Balle \cite{isermann1997trends} provide an overview of fault detection methods developed in the 1990s, including state and output observers, parity equations, bandpass filters, spectral analysis (fast Fourier transforms), and maximum-entropy estimation. 

For nonlinear systems, fault detection methods primarily use the concept of unknown input observability.  Controllability and observability Gramians for nonlinear systems are defined in \cite{hermann1977nonlinear}. De Persis and Isidori \cite{de2001geometric} develop a differential geometric method for fault detection and isolation. They use the concept of an unobservability subspace, based on the similar notion for linear systems (see \cite{isidori2013nonlinear}). The method guarantees the existence of a quotient subsystem of a given system space, which is only affected by the fault of interest. Martinelli \cite{martinelli2017nonlinear} develops a generalized algorithm to calculate the rank of the observable codistribution matrix (equivalent to the observability Gramian for linear systems) for nonlinear systems, and demonstrates its applicability for several practical examples, such as motion of a unicycle, a vehicle moving in three-dimensional space, and visual-inertial sensor fusion dynamics.\par

For a model-based fault detection problem, Maybeck et al. and Elgersma et al. used an assemble
of Kalman filters to match a particular fault pattern in \cite{maybeck1999multiple} and \cite{elgersma2001reconfigurable}, respectively. Boskovic et al. \cite{boskovic1999intelligent} and \cite{boskovic2001line} develop a multiple model method to detect and isolate actuator faults, using multiple hypothesis testing. In \cite{mcintyre2004fault}, a nonlinear observer-based fault identification method has been developed for a robot manipulator, which shows an asymptotic convergence of the fault observer to the actual fault value. Dixon et al. \cite{dixon2000fault} develop a torque filtering based fault isolation for a class of robotic manipulator systems. In \cite{bokor2009fault}, a model-based fault detection and identification approach is developed, by using a differential algebraic and residual generation method.\par 

Data-driven approaches such as \cite{ding2009subspace} and \cite{dong2011data}, use system data to identify the state-space matrices, without using any knowledge of system dynamics. 
In \cite{he2013least}, for a class of discrete time-varying networked systems with incomplete measurements, a least-squares filter paired with a residual matching (RM) approach is developed to isolate and estimate faults. This approach comprises several Kalman filters, with each filter designed to estimate the augment signal, composed of the system state and a specific fault signal, associated with it. An adaptive fault detection and diagnosis method is developed in \cite{lemos2013adaptive}, by implementing a clustering approach to detect faults. For incipient faults, Harmouche et al. in \cite{harmouche2014incipient} used a principal component analysis (PCA) framework to transform a data set with faulty operating conditions into either principal or residual subspaces. For nonlinear systems, although data-driven approaches are effective in many fault identification scenarios, the quality of fault detection greatly depends on the quality of available training data and the training data span. Zhang et al. \cite{zhang2015bayesian} proposed merging data-driven and model-based methods in a Bayesian framework. In \cite{bao2016improved}, sparse global-local preserving projections are used to extract sparse transformation vectors from given data set. The extracted sparse transformation is able to extract meaningful features from the data set, which results in a fault related feature extraction, as shown in \cite{bao2016improved}. 

Generative adversarial networks (GANs) were introduced in \cite{goodfellow2014generative} as data generative models in a zero-sum game framework. The training objective for a GAN is to increase the error rate of the discriminative network that was trained on an existing data set. Since their introduction, GANs have been used to augment machine learning techniques to do boosting of classification accuracy, generate samples, and detect fraud \cite{santurkar2017classification,tolstikhin2017adagan,arjovsky2017wasserstein,chongxuan2017triple,lin2017softmax,kilcher2017parametrizing,kodali2017train,zheng2018one,arici2016associative,saatci2017bayesian}. GAN has been proposed as an alternative to variational autoencoders \cite{44904,ulyanov2018takes}. Several research publications propose algorithms that can distinguish between ``true" samples and samples generated by GANs \cite{46638,46641,shen2017ae,samangouei2018defense}.  

The remainder of the paper is organized as follows. We provide a mathematical description of the fault detection problem along with the proposed approach in Section \ref{section2}. In Section \ref{sec_deep} we explain the architecture of an autoencoder and we propose a GAN to generate and classify data sets with normal and faulty operating conditions. A novel loss function, suitable for the proposed GAN based autoencoder network, is developed in Section \ref{sec_deep}.
In Section \ref{section7}, we first train and test a support vector machine (SVM) based classifier, on labeled data sets; (labeling is done based on both faulty and normal operating conditions). Subsequently, we demonstrate a way to improve the performance of the designed SVM, by training a GAN based autoencoder on a Gaussian random data set, which represents data sets with faulty operating conditions, for training the proposed GAN based network. In Section \ref{section7}, 
we show further improved performance of our proposed GAN based network architecture using a representative data set with faulty operating conditions generated by taking linear combinations of vectors that are orthogonal to the principal components of the normal data set space. Finally, we summarize our findings in Section \ref{conclusion}.

\section{The problem and the proposed method}
\label{section2}
\subsection{Problem description}
\label{prob_desc}

Figure \ref{fig:fig18} shows a schematic diagram of a closed-loop dynamical system. The dynamics of the system can be mathematically defined as
\begin{equation} \label{eq:system_dynamics}
\begin{split} 
\dot{x}=f(x,u,d) \\
y=g(x,u,d)
\end{split}
\end{equation}
\noindent where $x:[0,\infty)\rightarrow\mathbb{R}^{n}$ is the $n$-dimensional vector containing system states, $u:[0,\infty)\rightarrow\mathbb{R}^{m}$ is the $m$-dimensional vector containing control inputs, $d:[0,\infty]\rightarrow\mathbb{R}^{p}$ is the $p$-dimensional vector of exogenous disturbances, $f:\mathbb{R}^{n}\times\mathbb{R}^{m}\times\mathbb{R}^{p}\rightarrow\mathbb{R}^{n}$ is an unknown nonlinear mapping that represents the system dynamics, $y:[0,\infty]\rightarrow\mathbb{R}^{q}$ is a vector of  measurable system outputs, and $g:\mathbb{R}^{n}\times\mathbb{R}^{m}\times\mathbb{R}^{p}\rightarrow\mathbb{R}^{q}$ is an unknown nonlinear mapping that represents the relationship of input to output.
\begin{figure*}
\begin{centering}
\includegraphics[scale=0.40]{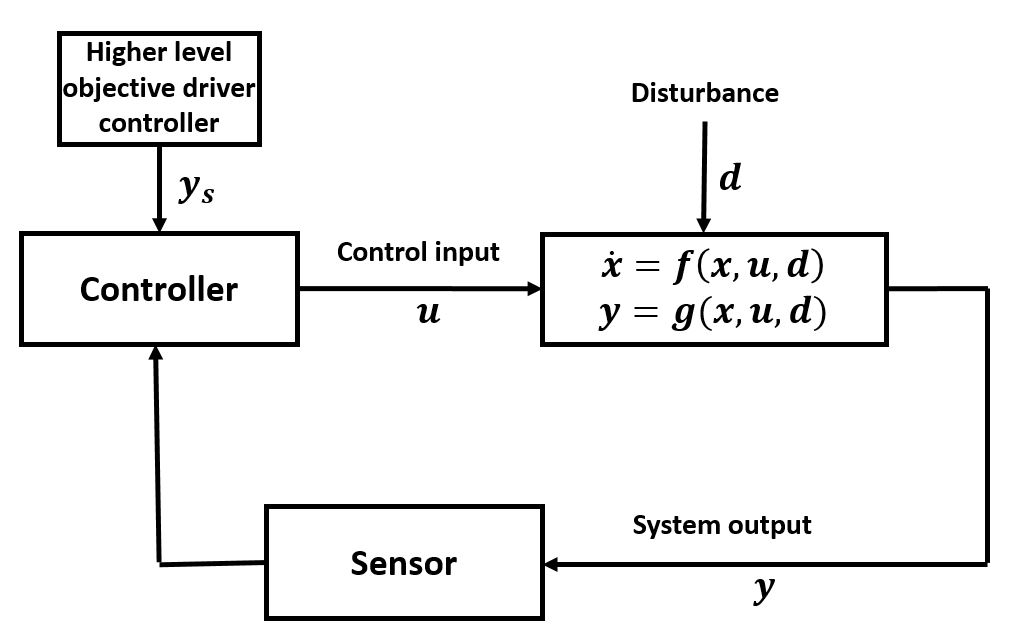}
\par\end{centering}
\caption{A closed-loop dynamical system.\label{fig:fig18}}
\end{figure*}

Now we define a fault detection problem for the dynamics in (\ref{eq:system_dynamics}) as follows. Given any data set $\mathcal{S}$ containing sample measurement pairs of $u$ and $y$, identify an unwanted change in the system dynamics. In order to further generalize the fault detection problem, we will only use the data set representing normal operating conditions. This restriction uses the fact that having a data set that incorporates faulty operating conditions indicates either having the capabilities of inserting system-level faults in the dynamics or having a known system dynamics $f$ (as in (\ref{eq:system_dynamics})). Developing either of these aforementioned capabilities involves manual labor and associated cost. We will further assume that the observable part of the system described in (\ref{eq:system_dynamics}) can be sufficiently identified from the available data set with normal operating conditions.

\subsection{Proposed method description}
\label{proposed_method}
For the fault detection problem described in Section \ref{prob_desc}, we develop a GAN based deep autoencoder, which uses data set with normal operating conditions (let us define this data space as $S_{0}$), to successfully identify the presence of faulty operating conditions in a given data set. In order to do that, we take the principal components of $S_{0}$ and use the orthogonals to those principal components to define a vector space $S_{1}$. Now, the training objective of our proposed GAN is to ``refine'' $S_{1}$, to calculate $S_{2}\subseteq S_{1}$, such that $S_{1}$ becomes a representative of the data set that contains system-level faulty operating conditions. The purpose of our deep autoencoder is to learn the data structure of $S_{0}$, by going through the process of encoding and decoding. Upon selecting an encoding dimension, we map the GAN generated space $S_{2}$ to the selected encoding dimension space. Let us designate the encoded representation of $S_{0}$ as $S_{0}^{*}$, and $S_{2}$ as $S_{2}^{*}$. Our final step is to design a classifier, which takes $S_{0}^{*}$ and $S_{2}^{*}$ for training. Furthermore, this entire training process, of both GAN and the deep autoencoder, is done simultaneously by defining a cumulative loss function. \par
In order to motivate the requirement of doing an orthogonal transformation on $S_{0}$ to define $S_{1}$, we demonstrate a case of defining $S_{1}$ using Gaussian random noises, and follow the aforementioned training process of the proposed network. Moreover, a single SVM based classifier is trained on labeled normal and fault-incorporated data sets, to compare performance with our proposed network for two different cases (orthogonal transformation and Gaussian-noise based prior selection).


\section{Proposed deep-learning based method}
\label{sec_deep}
\subsection{Overview of autoencoder networks}
\label{deep10}
Autoencoders are multilayer computational graphs used to learn a representation (encoding) for a set of data, for the purpose of dimensionality reduction or data compression. In other words, the training objective of our proposed deep autoencoder is to learn the underlying representation of a data set with normal operating conditions while going through the encoding transformations. A deep enough autoencoder, in theory, should be able to extract a latent representation signature from the training data, which can then be used to better distinguish normal and faulty operation. An autoencoder comprises two different transformations, namely encoding and decoding. The architecture of an autoencoder was first introduced and described by Bengio et al. in \cite{bengio2007greedy}. The encoder takes an input vector $\mathbf{x} \in\mathbf{R}^{d}$ and maps it to a hidden (encoded) representation $\mathbf{x}_{e}\in\mathbf{R}^{d^{\prime}}$, through a convolution of deterministic mappings. The decoder maps the resulting encoded expression into a reconstruction vector $\mathbf{x}^{\prime}$. We will use the notation $\mathcal{E}$ and $\mathcal{G}$ for the encoder and decoder of the autoencoder respectively.

Let the number of layers in the autoencoder network be $2n+1$, and let $\mathbf{y}_{i}$ denote the output for the network's $i^{th}$ layer. Then
\begin{gather}
\mathbf{y}_{0}=\mathbf{x}\nonumber, \mathbf{x}^{\prime}=\mathbf{y}_{2n+1} \\
\mathbf{y}_{i}=\sigma_{i}(\mathbf{w}_{i}^t \mathbf{y}_{i-1}+b_{i}) \nonumber,  \forall i\in\left[1,2n+1\right]. \label{deep1}
\end{gather}
Let $\boldsymbol{\theta}_{i}=\left\{\mathbf{w}_{i},b_{i}\right\}$ denote the parameters of the $i^{th}$ layer and $\sigma_{i}:\mathbf{R}\rightarrow\mathbf{R}$ be the activation function selected for each layer of the autoencoder. Let us also define $\boldsymbol{\theta}={\left\{\boldsymbol{\theta}_i\right\}}_{i=1}^{2n+1}$.

$\boldsymbol{\theta}$ defined for this autoencoder is optimized to minimize the average reconstruction error, given by
\begin{gather}
\label{training1}
\boldsymbol{\theta}=\mathrm{arg}\min_{\boldsymbol{\theta}}\frac{1}{m\sqrt{d}}\sum_{i=1}^{m}L(\mathbf{x}_{i},\mathbf{x}^{\prime}_{i})
\end{gather}
where $L$ is square of Euclidean distance, defined as $L(\mathbf{x},\mathbf{x}^{\prime})\triangleq \|\mathbf{x}-\mathbf{x}^{\prime}\|^{2}$, and $m\in\mathbb{N}$ is the number of available data points.

Our proposed autoencoder is trained on the normal data set, mentioned in Section \ref{dataset_building}. $90\%$ of the normal data (data span one year, with $5$ minute resolution) is used to train the autoencoder, and the rest $10\%$ is used for testing. Figure \ref{fig:fig1} shows both training and testing performance of our autoencoder with encoding dimension $100$, with increase in training epochs. Furthermore, selecting the proper encoding dimension is crucial for the following classifier to perform optimally. Figure \ref{fig:fig1} also demonstrates that the true positive accuracy rate from the classifier decreases when we decrease the encoding dimension. This signifies the loss of valuable information, if we keep decreasing the encoding dimension. For our application, we selected an encoding dimension of $100$.

In the next subsection, we give the formulation of our proposed generative model. Our generative model is in the spirit of the well-known GAN \cite{goodfellow2014generative}. Our proposed model will essentially generate samples that are not from the training data population. So clearly, unlike GAN, here the objective is not to fool the discriminator but to learn which samples are different. We will first formulate our proposed model and then comment on the relationship of our model with GAN in detail.

\subsection{Our proposed generative model}
We will use $\mathbf{x}_1$ to denote a sample of the data from the normal class, i.e., the class for which the training data is given. Let $p_{\text{data}}$ be the distribution of the normal class. We will denote a sample from the abnormal class by $\mathbf{x}_2$. In our setting, the distribution of the abnormal class is unknown, because the training data does not have any samples from the abnormal class. Our generative model will generate sample $\mathbf{x}_2$ from the unknown distribution, $p_{\text{noise}}$. {\it Note, here we will use the terminology ``data'' and ``noise'' to denote the normal and abnormal samples}. Let $p_z$ be the prior of the noise in the encoding space, i.e., $\mathbf{x}_2 \sim p_{\text{noise}} = \mathcal{G}(p_z)$. Furthermore, let $\mathcal{D}$ be the discriminator (a multilayer perceptron for binary classification), such that
\begin{align*}
\mathcal{D}(\mathbf{x})=\left\{\begin{array}{lr}
        1, & \text{if } \mathbf{x} \sim p_{\text{data}}\\
        0, & \text{if } \mathbf{x} \sim p_{\text{noise}}
        \end{array}\right.
\end{align*}
We will solve for $\mathcal{E}$, $\mathcal{G}$, and $\mathcal{D}$ in a maximization problem with the error function $V$ as follows:
\begin{align}
\label{deep:eq1}
V\left(\mathcal{D},\mathcal{E},\mathcal{G}\right) = \mathbf{E}_{\mathbf{x}_1\sim p_{\text{data}}} \left(\left(1-L(\mathbf{x}_1, \mathcal{G}(\mathcal{E}(\mathbf{x}_1)))\right) \right. \nonumber \\ + 
\left. \log(\mathcal{D}(\mathbf{x}_1))\right) + \mathbf{E}_{\mathbf{z}\sim p_z} \log(1-\mathcal{D}(\mathcal{G}(\mathbf{z})))
\end{align}
Note that here, $L$ is as defined in Eq. \ref{training1}. Furthermore, $L$ is normalized in $[0,1]$. 
Now, we will state and prove some theorems about the optimality of the solutions for the error function $V$.
\begin{theorem}
\label{deep:thm1}
For fixed $\mathcal{G}$ and $\mathcal{E}$, the optimal $\mathcal{D}$ is
\begin{align}
\label{deep:eq2}
\mathcal{D}^*(\mathbf{x}) = \frac{p_{\text{data}}(\mathbf{x})}{p_{\text{data}}(\mathbf{x})+p_{\text{noise}}(\mathbf{x})}
\end{align}
\end{theorem}
\begin{proof}
Given $\mathcal{G}$ and $\mathcal{E}$, $V\left(\mathcal{D},\mathcal{E},\mathcal{G}\right)$ can be written as:
\begin{align*}
\bar{V}\left(\mathcal{D}\right) &= \mathbf{E}_{\mathbf{x}_1\sim p_{\text{data}}} \left(\log(\mathcal{D}(\mathbf{x}_1))\right) + \mathbf{E}_{\mathbf{z}\sim p_z} \log(1-\mathcal{D}(\mathcal{G}(\mathbf{z}))) \\
&= \int_{\mathbf{x}} \left(p_{\text{data}}(\mathbf{x}) \log(\mathcal{D}(\mathbf{x}))+p_{\text{noise}}(\mathbf{x}) \log(1-\mathcal{D}(\mathbf{x}))\right)
\end{align*}
The above function achieves the maximum at $\mathcal{D}^*(\mathbf{x}) = \frac{p_{\text{data}}(\mathbf{x})}{p_{\text{data}}(\mathbf{x})+p_{\text{noise}}(\mathbf{x})}$.
\end{proof}
\begin{theorem}
\label{deep:thm2}
With $\mathcal{D}^*$ and fixed $\mathcal{E}$, the optimal $\mathcal{G}$ is attained when $\mathbf{x}_1 \sim p_{\text{data}}$ and $\mathbf{x}_2 \sim p_{\text{noise}}$ has zero mutual information.
\end{theorem}
\begin{proof}
Observe, from Eq. \ref{deep:eq1}, the first term is maximized if and only if the loss, $L$, is zero. Hence, when $\mathcal{D}=\mathcal{D}^*$ and $\mathcal{E}$ are fixed, the objective function, $V$ reduces to
\begin{align*}
\mathbf{E}_{\mathbf{x}_1\sim p_{\text{data}}} \left(1-L(\mathbf{x}_1, \mathcal{G}(\mathcal{E}(\mathbf{x}_1)))\right) + H(\mathbf{x}_1,\mathbf{x}_2) \\
- H(\mathbf{x}_1)-H(\mathbf{x}_2)
\end{align*}
The first term goes to zero, $1$, 
when the reconstruction is perfect, then, the remaining term is maximized {\it iff}, 
$$
H(\mathbf{x}_1,\mathbf{x}_2) - H(\mathbf{x}_1)-H(\mathbf{x}_2) = 0
$$
where $H(.)$ and $H(.,.)$ denote the marginal and joint entropies,  respectively. Note that, the LHS of the above expression is the mutual information, which is denoted by $I(\mathbf{x}_1, \mathbf{x}_2)$. Hence, the claim holds.
\end{proof}
Theorem \ref{deep:thm2} signifies that $\mathbf{x}_1 \sim p_{\text{data}}$ and $\mathbf{x}_2 \sim p_{\text{noise}}$ have zero mutual information, i.e., the distributions $p_{\text{data}}$ and $p_{\text{noise}}$ are completely uncorrelated. This is exactly what we intend to get, i.e., we want to generate abnormal samples that are completely different from the normal samples (training data). 
Now, we will talk about how to choose the prior $p_z$ after commenting on the contrast of our proposed formulation with \cite{goodfellow2014generative}. In GAN \cite{goodfellow2014generative}, the generator is essentially mimicking the data distribution to fool the discriminator. On the contrary, because our problem requires that samples be generated from outside training data, our proposed generator generates samples outside the data distribution. Note that one can choose the Wasserstein loss function in Eq. \ref{deep:eq1} similar to \cite{arjovsky2017wasserstein}. Below we will mention some of the important characteristics of our proposed model.
\begin{itemize}
\item Though we have called it a GAN based autoencoder, clearly the decoder $\mathcal{G}$ is generating the samples and hence acts as a generator in GAN.
\item In Equation \ref{deep:eq1}, on samples drawn from $p_{\text{data}}$, autoencoder (i.e., both encoder and decoder) acts, i.e., $\mathcal{G}(\mathcal{E}(\mathbf{x}_1))$ should be very closed 
to  $\mathbf{x}_1$, when $\mathbf{x}_1 \sim p_{\text{data}}$. On the contrary, on $\mathbf{z} \sim p_z$, only the decoder (generator $\mathcal{G}$) acts. 
Thus, the encoder is learned
only from $p_{\text{data}}$, while the decoder (generator) is learned 
from both $p_{\text{data}}$ and $p_z$. 
\item Unlike GAN, here we do not have a two player min max game,
instead we have a maximization problem over all the unknown parameters. Intuitively, this can be justified, because we are not generating counterfeit samples. 
\end{itemize}
\subsubsection{How to choose prior $p_z$}
\label{prior}
If we do not know anything about the structure of the data, i.e., about $p_{\text{data}}$, an obvious choice of prior for $p_z$ is a uniform prior. In this work, we have used PCA to extract the inherent lower-dimension subspace containing the data (or most of the data). This is essential not only for the selection of $p_z$ but for the selection of the encoding dimension as well. By the construction of our proposed formulation, the support of $p_z$ should be in the encoding dimension, i.e., in $\mathbf{R}^{d^{\prime}}$. Given the data, we will choose $d^{\prime}$ to be the number of principal directions along which the data has $>90\%$ variance. The span of these $d^{\prime}$ bases will give a point, $\mathcal{S}$, on the Grassmannian $\text{Gr}(d^{\prime},d)$, i.e., the manifold of $d^{\prime}$ dimensional subspaces in $\mathbf{R}^{d}$. The PCA suggests that ``most of the data'' lies on $\mathcal{S}$. In order to make sure that the generator generates $p_{\text{noise}}$ different from $p_{\text{data}}$, we will use the prior $p_z$ as follows.

Let $\mathcal{N} \in \text{Gr}(d^{\prime},d)$ be such that $\mathcal{N} \neq \mathcal{S}$. Let $\left\{\mathbf{n}_i\right\}_{i=1}^{d^{\prime}}$ be the bases of $\mathcal{N}$. We will say a sample $\mathbf{z} \sim p_z$ if $\mathbf{z}_i = \mathbf{x}^t\mathbf{n}_i$, for all $i$, for some $\mathbf{x} \sim p_{\text{data}}$. Without any loss of generality, assume $2d^{\prime}>d$; then, we can select the first $d-d^{\prime}$ $\mathbf{n}_i$s to be orthogonal to $\mathcal{S}$ (this can be computed by using Gram-Schmidt orthogonalization). The remaining $\left\{\mathbf{n}_i\right\}$s we will select from the bases of $\mathcal{S}$. 

\section{Results}
\label{section7}
In this section, we will present experimental validation of our proposed GAN based model. Recall that in our setting, we have only the ``normal'' samples in the training set and both ``normal'' and ``faulty'' samples in the testing set. In the training phase, we will use our proposed GAN based framework to generate samples from the population that are uncorrelated to the normal population. We will teach a  discriminator to do so. 
Then, in the testing phase, we will show that our trained discriminator can distinguish ``normal'' from ``faulty'' samples with high prediction accuracy. Furthermore, we will also show that using the prior, as suggested in Section \ref{deep10}, gives better prediction accuracy than the Gaussian prior.

\subsection{Dataset}
\label{dataset_building}
We use simulation data from a high-fidelity building energy system emulator. This emulator captures the building thermal dynamics, the performance of the building heating, ventilation, and air conditioning (HVAC), as well as the building control system. The control sequences that drive operation of the building HVAC are representative of typical existing large commercial office buildings in the U.S. We selected Chicago for the building location, and we used the typical meteorological year TMY3 data as simulation input. The data set comprises normal operation data and data representative of operation under five different fault types. We use these labeled data sets for training an SVM based classifier. The five fault types are the following: constant bias in outdoor air temperature measurement (Fault 1), constant bias in supply air temperature measurement (Fault 2), constant bias in return air temperature measurement (Fault 3), offset in supply air flow rate (Fault 4), and stuck cooling coil valve (Fault 5). Table \ref{data_whole} summarizes the characteristics of the data set including fault location, intensity, type, and data length.
\begin{table*}
\begin{centering}
\label{Table_data}
\resizebox{0.7\textwidth}{!}{%
\begin{tabular}{|>{\centering}p{2.8cm}|>{\centering}p{1.6cm}|>{\centering}p{1.8cm}|>{\centering}p{1.6cm}|c|}
\hline 
\textbf{Faulty Component} & \textbf{System} & \textbf{Time of Year} & \textbf{Fault Intensity} & \textbf{Data Length}\tabularnewline
\hline
\hline 
- & Building HVAC & Jan-Dec & - & Yearly\tabularnewline
\hline 
Outdoor air temperature sensor & Mid-floor AHU & Feb, May, Aug, Nov & $\pm2,\pm4$ $^{\circ}$C & Monthly\tabularnewline
\hline 
Supply air temperature sensor & Mid-floor AHU & Aug & $-2$ $^{\circ}$C & Monthly\tabularnewline
\hline 
Return air temperature sensor & Mid-floor AHU & May-Jun & $+4$ $^{\circ}$C & Monthly\tabularnewline
\hline 
Supply air flow rate set point & Mid-floor AHU & May-Jun & $-0.1$ kg/s  & Monthly\tabularnewline
\hline 
Cooling coil valve actuator & Mid-floor AHU & Aug & $25\%$, $50\%$ & Monthly\tabularnewline
\hline 
\end{tabular}}
\par\end{centering}
\caption{Data set includes normal and fault scenarios sampled at 1 minute resolution. AHU is air handling unit.}\label{data_whole}
\end{table*}
\subsection{Application of SVM on simulated dataset}
Support vector machines are statistical classifiers originally introduced by \cite{boser1992training} and \cite{vapnik1998statistical}, later formally introduced by \cite{cristianini2000introduction}. In this subsection, we will briefly demonstrate the use of SVMs for classifying properly labeled datasets with normal and various faulty operating conditions. SVM separates a given set of binary labeled training data with a hyperplane, which is at maximum distance from each binary label. Therefore, the objective of this classification method is to find the maximal margin hyperplane for a given training data set. For our work, a linear separation is not possible (i.e., to successfully draw a line to separate faulty and normal data sets); that motivates the necessity of using a radial basis function (RBF) kernel (\cite{du2014radial}), along with finding a non-polynomial hyperplane to separate the labeled datasets.\par 
Before describing SVM classification in detail, the RBF kernel (see \cite{vapnik1998statistical}) on two samples  $x_{i}$ and $x_{j}$ is defined as
\begin{gather}
K_{ij}\triangleq K(x_{i},x_{j})=\mathrm{exp}\Bigl(-\frac{\|x_{i}-x_{j}\|^{2}}{2\sigma^{2}}\Bigr),\label{deep4}
\end{gather}
where $\|x_{i}-x_{j}\|$ denotes the square of the Euclidean distance, and $\sigma$ is a user-defined parameter, selected to be unity for this work.\par
A Scikit learning module available in Python 3.5+ is used for implementation of SVM on the building HVAC data set. Specifically, NuSVC is used with a cubic polynomial kernel function to train for normal and faulty data classification. As the nu value represents the upper bound on the fraction of training error, a range of nu values from $0.5$ to $0.9$ are tried during cross validation of the designed classifier. Table \ref{Table1} shows the confusion matrix for the designed SVM classifier, for data sets labeled ``normal" and ``fault type 1," where the true positive accuracy rate is less than $50\%$. This finding, as mentioned before, justifies the need to develop an adversarial based classifier, which uses the given normal data to create representative faulty training dataset.
\begin{center}
\captionof{table}{Confusion matrix for simulated building environment data set using SVM classifier} 
\label{Table1}
\begin{tabular}{l|l|c|c|c}
\multicolumn{2}{c}{}&\multicolumn{2}{c}{\textbf{True diagnosis}}&\\
\cline{3-4}
\multicolumn{2}{c|}{}&Normal&Fault 1&\multicolumn{1}{c}{Total}\\
\cline{2-4}
\multirow{2}{*}{\rotatebox{90}{\textbf{Prediction}}}& Normal & \cellcolor{blue!25}$43.27\%$ & $56.73\%$ & $100\%$\\
\cline{2-4}
& Fault & $52.18\%$ & \cellcolor{blue!25}$47.82\%$ & $100\%$\\
\cline{2-4}
\end{tabular}
\end{center}
\vspace{20pt}

\subsection{Data set using Gaussian noise}
In Table \ref{confustion matrix1}, the confusion matrix is shown for the GAN based autoencoder, where Gaussian noise is used as an input to GAN, for representing a training class of fault types for the GAN based autoencoder. From left to right, the values in Table \ref{confustion matrix}, 
denote true positive rate (TPR), false positive rate (FPR), false negative rate (FNR), and true negative rate (TNR). Although in Table \ref{confustion matrix1}, the normal data set gives more than $90\%$ TPR, the faulty data set gives around $40\%$ TPR. We can conclude that Gaussian noise as an initial representative of a faulty data set does not represent a completely different faulty data set from the normal data set. 
\begin{center}
\captionof{table}{Confusion matrix for building data set using Gaussian-noise generator} 
\label{confustion matrix1}
\begin{tabular}{l|l|c|c|c}
\multicolumn{2}{c}{}&\multicolumn{2}{c}{\textbf{True diagnosis}}&\\
\cline{3-4}
\multicolumn{2}{c|}{}&Normal&Fault&\multicolumn{1}{c}{Total}\\
\cline{2-4}
\multirow{2}{*}{\rotatebox{90}{\textbf{Prediction}}}& Normal & \cellcolor{blue!25}$93.10\%$ & $6.90\%$ & $100\%$\\
\cline{2-4}
& Fault & $58.40\%$ & \cellcolor{blue!25}$41.60\%$ & $100\%$\\
\cline{2-4}
\end{tabular}
\end{center}
\vspace{20pt}

\subsection{Data set using PCA and orthogonal projection}
Table \ref{confustion matrix} shows the confusion matrix for the same high fidelity building data set, where PCA is used to generate an initial representative of a fault-incorporated data set as in Section \ref{prior}. For the sake of completeness, a few other confusion matrix terms are also calculated as follows: TPR is $92.30\%$, TNR is $72.80\%$, FPR is $22.76\%$,
FNR is $27.20\%$, accuracy (ACC) is $82.55\%$, positive predictive value (PPV) is $77.24\%$, negative predictive value (NPV) is $90.43\%$, false discovery rate (FDR) is $22.76\%$, and finally, false omission rate (FOR) is $9.57\%$. Table \ref{confustion matrix} shows better classification performance, compared to both the other methods described before.
\begin{center}
\captionof{table}{Confusion matrix for building data set using PCA transformation} 
\label{confustion matrix}
\begin{tabular}{l|l|c|c|c}
\multicolumn{2}{c}{}&\multicolumn{2}{c}{\textbf{True diagnosis}}&\\
\cline{3-4}
\multicolumn{2}{c|}{}&Normal&Fault&\multicolumn{1}{c}{Total}\\
\cline{2-4}
\multirow{2}{*}{\rotatebox{90}{\textbf{Prediction}}}& Normal & \cellcolor{blue!25}$92.30\%$ & $7.70\%$ & $100\%$\\
\cline{2-4}
& Fault & $27.20\%$ & \cellcolor{blue!25}$72.80\%$ & $100\%$\\
\cline{2-4}
\end{tabular}
\end{center}
\vspace{20pt}

\begin{figure*}
\begin{center}
\includegraphics[scale=0.6]{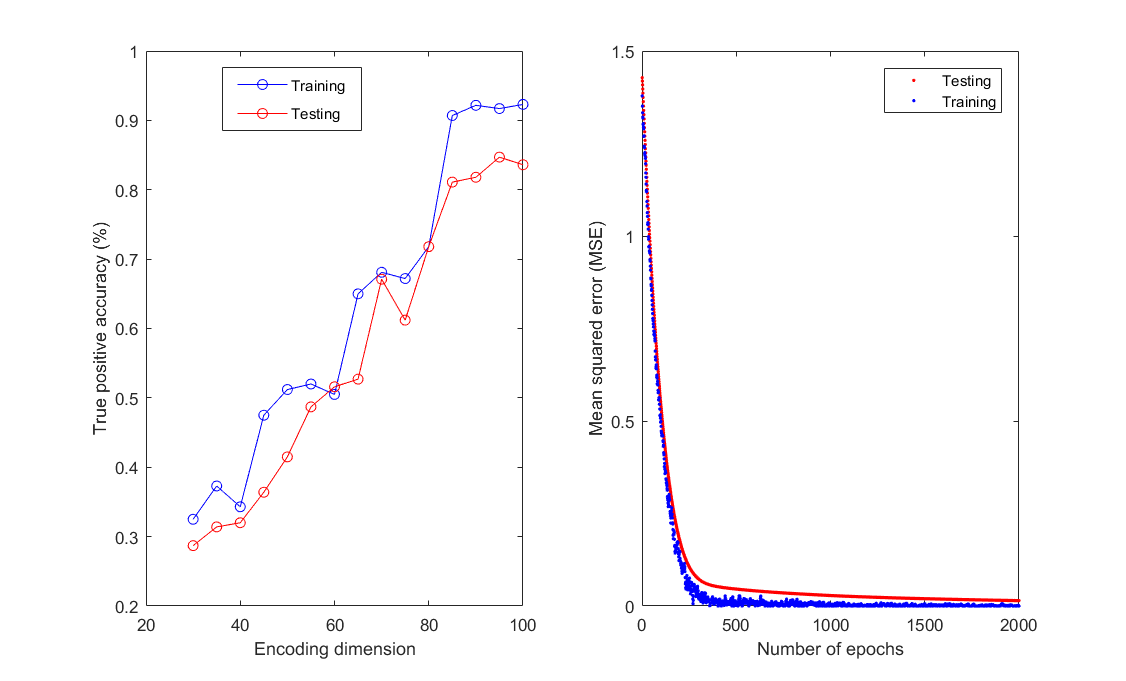}
\end{center}
\caption{Change in true positive accuracy with change in encoding dimension (left); Training and testing performance of the proposed autoencoder with encoding dimension  $100$ (right) \label{fig:fig1}}
\end{figure*}

\subsection{Comparison of results}
We demonstrated three different methods for the fault detection problem, applied to a high complexity building data set. The SVM classifier, despite using labeled data sets, gives poor TPRs for our data set. Our proposed GAN based deep autoencoder network is trained and tested using two different training approaches. First, we use a Gaussian-noise based data set as a representation of space $S_{1}$ (as in Section \ref{section2}) to train the designed GAN and simultaneously find a representative class for a data set with faults, i.e., $S_{2}^{*}$. Although the Gaussian-noise based data set gives much better TPR for the normal data set than the SVM, it performs poorly when identifying a data set with faulty conditions. Second, we use orthogonal transformation on the normal data set to generate $S_{2}$, and subsequently our proposed GAN based autoencoder is trained on this new $S_{2}$ to generate $S_{2}^{*}$. Although the orthogonal transformation based training approach gives similar TPRs for the normal data set, it gives significantly better performance for the data set with faulty conditions than the Gaussian-noise based training approach.
\subsection{Group testing}
In this section, we will do some statistical analysis of the output produced by our proposed framework. More specifically, we will do group testing in the encoding space, i.e., we will pass the generated noise and the data through the trained encoder and perform a group test. However, because we do not know the distribution of the data and noise in the encoding space, we cannot do a two-sample t-test. We will develop a group testing scheme for our purpose. Let $\left\{\mathbf{y}^1_i\right\}$ and $\left\{\mathbf{y}^2_i\right\}$ be two sets of samples in the encoding space generated using our proposed network. Let $\left\{C^1_i  :=\mathbf{y}^1_i\left(\mathbf{y}^1_i\right)^t\right\}$ and $\left\{C^2_i:=\mathbf{y}^2_i\left(\mathbf{y}^2_i\right)^t\right\}$ 
be the corresponding covariance matrices capturing the interactions among dimensions. We will identify each of the covariance matrices with the product space of Stiefel and symmetric positive definite (SPD) matrices, as proposed in \cite{bonnabel2009riemannian}. 


Now, we perform the kernel based two-sample test to find the group difference \cite{gretton2009fast} between $\left\{C^1_i\right\}$ and $\left\{C^2_i\right\}$. In order to use their formulation,
we first define the intrinsic metric we will use in this work. We will use the general linear (GL)-invariant metric for SPD matrices, which is defined as follows: Given $X, Y$ as two SPD matrices, the distance, $d(X,Y) = \sqrt{\text{trace}\left(\left(Log\left(X^{-1}Y\right)\right)^2\right)}$. For the Stiefel manifold, we will use the canonical metric \cite{kaneko2013empirical}. On the product space, 
we will use the $\ell_1$ norm as the product metric. As the kernel, we will use the Gaussian RBF, which is defined as follows: Given $C_1=(A, X)$ and $C_2=(B, Y)$ as two points on the product space, 
the kernel, $k\left(C_1, C_2\right):= 
\exp\left(-d^2\left(C_1,C_2\right)\right)$. Here, $d$ is the product metric. Given $\left\{C^1_i\right\}_{i=1}^{N_1}$ and $\left\{C^2_i\right\}_{i=1}^{N_2}$, the maximum mean discrepancy (MMD) is defined as follows:
\begin{align}
\text{MMD}\left(\left\{C^1_i\right\}, \left\{C^2_i\right\}\right)^2 = \frac{1}{N_1^2} \sum_{i,j}k\left(C^1_i,C^1_j\right)-\nonumber \\ \frac{2}{N_1N_2}\sum_{i,j}k\left(C^1_i,C^2_j\right)+\frac{1}{N_2^2} \sum_{i,j}k\left(C^2_i,C^2_j\right)
\end{align}

For a level $\alpha$ test, we reject the null hypothesis $H_0 = \left\{\right.$samples from the two groups are from same distribution$\left.\right\}$ if $MMD < 2 \sqrt{1/\max{N_1,N_2}} \left(1+\sqrt{-\log \alpha}\right)$. Finally, we conclude from the experiments that for our proposed framework, we reject the null hypothesis with $95\%$ confidence. 

\section{Conclusion}
\label {conclusion}
A novel GAN based autoencoder is introduced in this paper. This proposed network performs very well when compared to an SVM based classifier. Although the SVM classifier uses labeled training data for classification, it still gives less than $50\%$ TPR for our high complexity simulated data set. On the other hand, the proposed GAN based deep autoencoder gives significantly better performance for two different types of training scenarios. The proposed GAN based autoencoder is initially trained on a random Gaussian data set. Next, orthogonal projection is used to generate a data set that is perpendicular to the given normal data set. This orthogonally projected data set is used as an initial fault-incorporated data set for our proposed GAN based autoencoder for training. Confusion matrices for both training scenarios are presented, and both of them perform very well compared to the SVM based classification approach. Finally, a statistical group test demonstrates that our encoded normal and GAN based fault-incorporated data spaces (i.e., data sets in $S_{0}^{*}$ and $S_{2}^{*}$ spaces, respectively) are statistically different, and subsequently validates the favorable performance of our proposed network.

\bibliographystyle{unsrt}

\bibliography{main.bib}
\end{document}